\documentclass[accepted]{uai2025} % after acceptance, for a revised version; 
\usepackage[american]{babel}
\usepackage{natbib} % has a nice set of citation styles and commands
\usepackage{mathtools} % amsmath with fixes and additions
\usepackage[dvipsnames]{xcolor}
\usepackage{graphicx}
\usepackage{balance}
\usepackage{amsmath,amssymb,amsfonts}
\usepackage{textcomp}
\usepackage{booktabs}
\usepackage{multirow}

\usepackage{todonotes}
\usepackage{amsthm}
\usepackage{tabularx}
\usepackage{subfig}
\usepackage{graphicx}
\newcommand{\bm}[1]{\mathbf{#1}}

\newcommand{\bx}{\bm{x}}
\newcommand{\bX}{\mathbf{X}}
\newcommand{\by}{\bm{y}}

\usepackage{algorithmic}
\usepackage{color, colortbl}
\usepackage{sidecap}
\usepackage{rotating}
\usepackage{hyperref}
\usepackage{comment}
\usepackage[dvipsnames]{xcolor}
\newcommand{\croppdf}[1]{\IfFileExists{#1-crop.pdf}{\immediate\write18{pdfcrop #1.pdf}}{\immediate\write18{pdfcrop #1.pdf}}}
\hypersetup{
 colorlinks=true,
 linkcolor=blue,
 urlcolor=blue,
 citecolor=blue
}
\newtheorem{theorem}{Theorem}
\newtheorem{lemma}{Lemma}
\newtheorem{proposition}[theorem]{Proposition}

\definecolor{myblue}{RGB}{50, 116, 161}
\definecolor{myorange}{RGB}{225, 129, 44}
\definecolor{mygray}{RGB}{128, 128, 128}
\DeclareMathOperator*{\argmin}{arg\,min}

\newenvironment{highlight}{}{}
\newcommand{\update}[1]{#1}

\newcommand{\remove}[1]{}
\newcommand{\removetwo}[1]{}
\newcommand{\removecref}[1]{}
\newcommand{\removecreftwo}[1]{}
\newcommand{\notetwo}[1]{}
\newcommand{\note}[1]{}

\title{Bayesian Optimization over Bounded Domains with the Beta Product Kernel}

\author[1]{\href{mailto:huy.nguyen@oulu.fi}{Huy~Hoang~Nguyen}{}}
\author[2]{Han~Zhou}
\author[2]{Matthew~B.~Blaschko$^{*}$}
\author[1]{Aleksei~Tiulpin\thanks{Equal last author}}
\affil[1]{%
    Research Unit of Health Sciences and Technology, Finland\\
    University of Oulu\\
    Oulu, Finland
}
\affil[2]{%
    Department ESAT, Center for Processing Speech and Images\\
    KU Leuven\\
    Leuven, Belgium
}
  
\begin{document}
\maketitle

\begin{abstract}
  Bayesian optimization with Gaussian processes (GP) is commonly used to optimize black-box functions. The Matérn and the Radial Basis Function (RBF) covariance functions are used frequently, but they do not make any assumptions about the domain of the function, which may limit their applicability in bounded domains. To address the limitation, we introduce the Beta kernel, a non-stationary kernel induced by a product of Beta distribution density functions. Such a formulation allows our kernel to naturally model functions on bounded domains. 
  We present statistical evidence supporting the hypothesis that the kernel exhibits an exponential eigendecay rate, based on empirical analyses of its spectral properties across different settings.
  Our experimental results demonstrate the robustness of the Beta kernel in modeling functions with optima located near the faces or vertices of the unit hypercube. The experiments show that our kernel consistently outperforms a wide range of kernels, including the well-known Matérn and RBF, in different problems, including synthetic function optimization and the compression of vision and language models. \update{Our implementation is available at~\url{https://github.com/imedslab/BetaKernel}.}
\end{abstract}

\section{INTRODUCTION}

Bayesian optimization (BO) is a well-founded approach for global optimization of noisy black-box functions, which are often expensive to evaluate. At its core, BO relies on a surrogate model to approximate the objective function and guide the search process efficiently. A Gaussian process (GP) is commonly used as this surrogate due to its flexibility, ability to quantify uncertainty, and capability to incorporate prior knowledge through covariance functions (also called kernels). 

\begin{figure}[t]
    \centering
    \croppdf{figures/matern_2.5_convar_3d}
    \croppdf{figures/rbf_1_convar_3d}
    \croppdf{figures/buck_0.1_convar_3d}
    % \croppdf{figures/buck_0.25_convar_3d}
    % \croppdf{figures/buck_0.75_convar_3d}
    \croppdf{figures/buck_1.5_convar_3d}    
    \subfloat[\small RBF kernel\label{fig:rbf_1d}]{
    \includegraphics[width=0.48\linewidth]{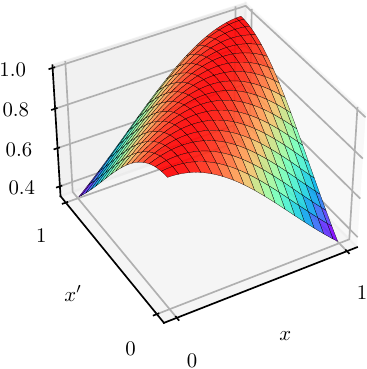}}
    \hfill
    \subfloat[\small Matérn kernel ($\nu=\frac{5}{2}$)\label{fig:matern_1d}]{
    \includegraphics[width=0.48\linewidth]{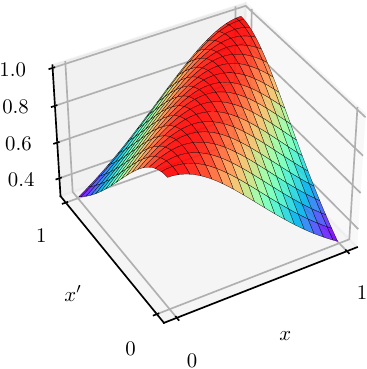}}
    \\
    % \subfloat[\small Beta kernel ($h=0.25$)\label{fig:buc_1d_25}]{
    % \includegraphics[width=0.48\linewidth]{figures/buck_0.25_convar_3d-crop.pdf}}
    \subfloat[\small Beta kernel ($h=0.1$)\label{fig:buc_1d_10}]{
    \includegraphics[width=0.48\linewidth]{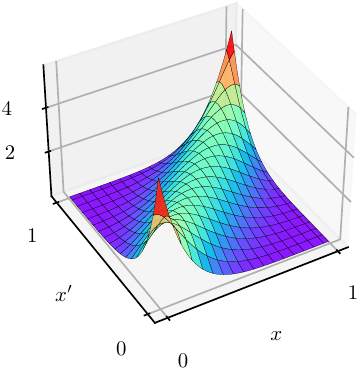}}
    \hfill
    \subfloat[\small Beta kernel ($h=1.5$)\label{fig:buc_1d_150}]{
    \includegraphics[width=0.48\linewidth]{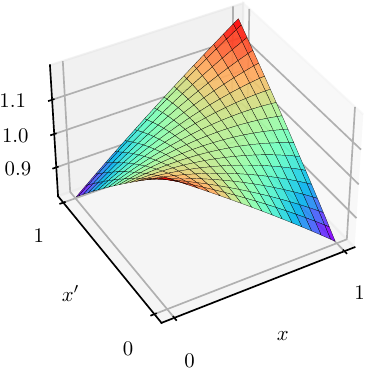}}
    \caption{\small Covariance matrices of the Matérn kernel and our Beta kernel in the unit 1D domain. Different from the Matérn kernel, the variation along the diagonal indicates the non-stationarity of our kernel.}
    \label{fig:1d_covar}
\end{figure}

A GP is defined by its mean and covariance function. The choice of kernel is critical in encoding prior knowledge about the target function's behavior. Among many available kernels~\citep{oh2018bock,wilson2013gaussian,duvenaud2011additive,jebara2004probability}, the Matérn and Radial Basis Function (RBF) kernels have been the most extensively studied~\citep{stein2012interpolation,williams2006gaussian,srinivas2009gaussian,santin2016approximation,vakili2021information} and are widely adopted in practical applications~\citep{pedregosa2011scikit,head2018scikit,gardner2018gpytorch}~\footnote{The Matérn kernel is the default choice in various GP libraries such as Scikit-optimize, GPyTorch, and GPyOpt.} due to their flexibility and smoothness properties. The RBF kernel, with its infinitely differentiable form, is ideal for modeling smooth functions, while the Matérn kernel provides adjustable smoothness through the $\nu$ parameter. Formally, the Matérn kernel is defined as~\citep{rasmussen2003gaussian,myers1994spatial} 
\begin{align} 
K_{\text{Matérn}}(r) = \frac{2^{1-\nu}}{\Gamma(\nu)} \left(\sqrt{2\nu} \frac{r}{\ell}\right)^{\nu} K_{\nu}\left(\sqrt{2\nu} \frac{r}{\ell}\right), 
\end{align}
where $r= \| \bm{x} - \bm{x}' \|_2$, $\nu > 0$ is a smoothness parameter, $\ell$ is a positive length scale, $\Gamma(\cdot)$ is the Gamma function, and $K_{\nu}$ is a modified Bessel function~\citep{abramowitz1968handbook}. When $\nu \rightarrow \infty$, the Matérn kernel is equivalent to the RBF kernel, formulated as
\begin{align}
    K_{\text{RBF}}(r) = \exp \left (-\frac{r^2}{2\ell^2} \right ).
\end{align}
RBF and Matérn are defined on unbounded domains. However, in most practical applications, the objective function is specified on bounded ones. The unawareness of the boundary of these kernels may result in either over-exploration (see Figure~\ref{fig:rbf_convergence}) or neglecting boundary regions (see Figure~\ref{fig:matern_convergence}).

Both the kernels, together with the Newton-Girard Additive (NGA)~\citep{duvenaud2011additive} and Spectral Mixture (SM)~\citep{wilson2013gaussian} kernels, are isotropic and stationary since they can be expressed as $ K(\mathbf{x}, \mathbf{x}') = 
K(\|\mathbf{x} - \mathbf{x}'\|), \forall \mathbf{x}, \mathbf{x}'$. However, it has been noted by~\cite{oh2018bock} that such stationary kernels face challenges with the boundary issue introduced by~\cite{swersky2017improving}, primarily due to their lack of awareness regarding absolute locations. To address this limitation, they proposed BOCK with a cylindrical (CYL) kernel defined in a hypersphere. Nevertheless, besides the CYL kernel, the high performance of BOCK is also attributed to the input wrapping proposed by~\cite{snoek2014input}, which is based on the cumulative density function (CDF) of the Beta distribution.

In this study, we propose a non-stationary Beta distribution-based kernel, thus the name \textbf{Beta kernel}, specifically defined on unit hypercubes. Our kernel is induced from products of multiple Beta distribution density functions, each of which naturally represents a wide range of functions defined on $[0,1]$. We empirically show that the eigenvalue decay rate of our kernel is exponential, which is similar to RBF's. 

We argue that the GP benchmarks are typically performed on synthetic test functions with optima located near the center of the search space~\citep{laguna2005experimental,molga2005test,back1996evolutionary,dixon1978global,pmlr-v235-hvarfner24a}. In a high-dimensional unit hypercube, while the boundary volume becomes significantly larger than the central volume, the volumes near the vertices are significantly small (see the example in Sec.~\ref{sc:unit_hypercube_ex}). Consequently, while locating optima near the boundary is more challenging than in the central region, it is unlikely to sample data points near the vertices. To evaluate this, we modify the test function domains so that the optima are positioned near a face or vertex of the unit hypercube. Our results demonstrate that the proposed kernel is more robust than a wide range of baselines -- including RBF and Matérn -- under these boundary-focused settings across different test functions.

In addition, we conduct experiments on model compression tasks for deep vision and language models, including Vision Transformer (ViT), Bidirectional encoder representations from transformer (BERT), Generative Pre-trained Transformer 2 (GPT-2), and Decoding-enhanced BERT with disentangled attention (DeBERTa). The results show that our kernel substantially outperforms the baselines across the compression tasks. Furthermore, our experiments also demonstrate that the Beta kernel consistently surpasses the Matérn kernel -- the most competitive baseline -- when combined with various acquisition functions.

\section{BETA KERNEL}

\begin{highlight}
\subsection{Preliminaries}
\end{highlight}

\paragraph{Gaussian Process.} A GP is defined as $f \sim \text{GP}(\mu(\bm{x}), K(\bm{x}, \bm{x}'))$, with its mean function $\mu(\bm{x}) = \mathbb{E}[f(\bm{x})]$, and covariance (kernel) function $K(\bm{x}, \bm{x}')=\mathbb{E}[(f(\bm{x}) - \mu(\bm{x}))(f(\bm{x}') - \mu(\bm{x}'))]$. Consider a setting in which we are given a set of observations $\bX_t=(\bx_1, \dots, \bx_t)$ and corresponding noisy outputs $\by_t = (y_1, \dots, y_t)^\top$, where $y_t=f(\bx_t) + \varepsilon_t$, and $\varepsilon_t \sim \mathcal{N}(0, \sigma^2)$ are i.i.d. Gaussian noise. Then, the posterior over $f$ conditioned on the observations is also a GP whose mean $\mu_t(\bx)$, covariance $K_t(\bx, \bx')$, and variance $\sigma^2_t(\bx)$ are formulated as
\begin{align}
    \mu_t(\bx) &= \mathbf{K}_t(\bx)^\top (\mathbf{K}_t^* + \sigma^2\mathbb{I})^{-1} \by_t, \nonumber \\
    K_t(\bx, \bx') &= K(\bx, \bx') - \mathbf{K}_t(\bx)^\top (\mathbf{K}_t^* + \sigma^2 \mathbb{I})^{-1} \mathbf{K}_t(\bx'), \nonumber \\
    \sigma^2_t(\bx) &= K_t(\bx, \bx),
    \label{eq:gp}
\end{align}
where $\mathbf{K}_t(\bx) = [K(\bx_1, \bx), \dots, K(\bx_t, \bx)]^\top$, and $\mathbf{K}_t^*=[K(\bx_i, \bx_j)]_{i,j=1}^t$ is the positive definite covariance matrix.

\paragraph{Beta Distribution.} The probability density function of the Beta distribution $\text{Beta}(\alpha,\beta)$ is defined on $[0,1]$ as
\begin{align}
    \update{\frac{1}{B(\alpha, \beta)}} x^{\alpha-1}(1-x)^{\beta-1},
\end{align}
where $x \in [0,1]$, $\alpha, \beta > 0$ \begin{highlight} and 
\begin{equation}
\begin{split}
    B(\alpha, \beta) = \int_0^1 s^{\alpha- 1 }(1-s)^{\beta - 1}ds =\frac{\Gamma(\alpha)\Gamma(\beta)}{\Gamma(\alpha + \beta)},
\end{split}
\label{eq:beta_func}
\end{equation}
\end{highlight}
\update{is the beta function with $\Gamma(\cdot)$ denoting the gamma function.} In this work, we are interested in the case where $\alpha, \beta > 1$, which allows the mode to be well-defined as $\frac{\alpha - 1}{\alpha + \beta - 2}$.

\begin{highlight}
\paragraph{Probability product kernels.} Let $p$ and $p'$ denote probability distributions on a space $\mathcal{S}$, and let $\rho$ be a positive constant.  
Assuming that $p^\rho, {p'}^\rho \in L^2(\mathcal{S})$, the probability product kernel between $p$ and $p'$ is defined by~\cite{jebara2004probability} as
\begin{align}
    K^{\text{prob}}(p, p') = \int_\mathcal{S} p(s)^\rho p'(s)^{\rho} ds.
\end{align}
When $\rho=1$, the kernel simplifies to the expectation of one distribution with respect to the other. In this case, it becomes equivalent to a kernel defined over two corresponding samples $\bx$ and $\bx'$ drawn from the distributions $p$ and $p'$, respectively:
\begin{align}
    K(\bx, \bx')= \int p(s) p'(s) ds = \mathbb{E}_p[p'(s)] = \mathbb{E}_{p'}[p(s)].
\end{align}
\end{highlight}
\begin{highlight}
\subsection{Kernel Derivation}

Let us first consider the one-dimensional case. Specifically, $\forall x, x' \in [0,1]$, the Beta product kernel (or simply, \emph{Beta kernel}), is constructed as a probability product kernel~\citep{jebara2004probability} with respect to the Beta distribution as follows
\begin{align}
    K_{\beta}(x, x')&=K_{\beta}((\alpha, \beta), (\alpha', \beta')) \nonumber \\
    &=C\int_0^1s^{\alpha -1}(1-s)^{\beta -1}s^{\alpha' -1}(1-s)^{\beta' -1}\mathrm{d}s \nonumber\\
    &=C\int_0^1s^{\alpha + \alpha'-2}(1-s)^{\beta + \beta' -2}\mathrm{d}s, \nonumber
\end{align}
where $x$ and $x'$ represent the modes of the two Beta distributions, respectively, and the normalization term $C = \frac {1}{\mathrm {B} (\alpha ,\beta)} \cdot \frac {1}{\mathrm {B} (\alpha' ,\beta' )}$. The shape parameters $\alpha$ and $\beta$ are connected to the modes via a common smoothing parameter $h$, such that $\alpha= 1 + \frac{x}{h}$ and $\beta = 1 + \frac{1 - x}{h}$. To ensure the unique existence of the modes $x$ and $x'$, we require that $\alpha, \beta > 1$, which implies $h > 0$.

Based on the definition of the Beta function in Eq.~\eqref{eq:beta_func} and given that $\alpha+\beta = \alpha' + \beta'=2 +\frac{1}{h}$, it follows that
\begin{align}
    &K_\beta(x, x') = \frac {\mathrm {B} (\alpha + \alpha' -1, \beta + \beta' -1)}{\mathrm {B} (\alpha ,\beta )\mathrm {B} (\alpha' ,\beta')} \nonumber \\
    &= \frac{\Gamma(\alpha + \beta)\Gamma(\alpha' + \beta')\Gamma(\alpha + \alpha' -1 )\Gamma(\beta + \beta' -1)}{\Gamma(\alpha)\Gamma(\beta)\Gamma(\alpha')\Gamma(\beta')\Gamma(\alpha + \alpha' + \beta + \beta' -2)} \nonumber \\
    &= \frac{\Gamma^2(\frac{1}{h} + 2)}{\Gamma(\frac{2}{h} + 2)} \frac{\Gamma(\alpha + \alpha' -1 )\Gamma(\beta + \beta' -1)}{\Gamma(\alpha)\Gamma(\beta)\Gamma(\alpha')\Gamma(\beta')}. \nonumber
\end{align}
By assuming that all dimensions are independent of each other, the definition of the Beta kernel in the $d$-dimensional space can be easily extended as the product over $d$ dimensions. Specifically, $\forall \bx, \bx' \in [0,1]^d$, we define the Beta kernel as
\begin{align}
     K_{\beta}(\bm{x},\bm{x}') = \Tilde{C} \prod_{i=1}^d\frac{\Gamma(\alpha_i + \alpha'_i - 1)\Gamma(\beta_i + \beta'_i - 1)}{\Gamma(\alpha_i)\Gamma(\beta_i)\Gamma(\alpha'_i)\Gamma(\beta'_i)},
     \label{eq:beta_kernel}
\end{align}
where
\begin{align}
     \Tilde{C} = \prod_{i=1}^d\frac{\Gamma^2 \left (\frac{1}{h_i} +2 \right ) }{\Gamma \left (\frac{2}{h_i} + 2 \right )}.
\end{align}
\end{highlight}

% \begin{comment}

% \end{comment}

\update{\subsection{Properties}}
\begin{highlight}
\paragraph{Validity of Beta product kernel}
To ensure that the Beta product kernel is a valid kernel function, it must be positive semidefinite. This property guarantees that the resulting covariance matrix is symmetric and all its eigenvalues are non-negative, which is essential for its use in kernel-based methods such as support vector machines and Gaussian processes. Fortunately, this property follows from a more general result concerning probability product kernels. In particular, Theorem~\ref{eq:theorem_ppk_psd} guarantees the validity of any kernel constructed as a probability product kernel, including the Beta product kernel.

\begin{theorem}
    $K(\bx, \bx') = \int_\mathcal{S} p(s \mid \bx)^\rho p(s\mid \bx')^\rho \mathrm{d}s$ is positive semidefinite for \(\rho > 0\).
\begin{proof}
    We follow the definition of a Mercer kernel and consider all possible $\bm{x}_1, \dots, \bm{x}_m$ and real numbers $c_1, \dots, c_m$. To show that $K(\bx, \bx')$ is a valid kernel, we need to prove that $\sum_{i=1}^m \sum_{j=1}^{m} c_i c_j K(\bm{x}_i, \bm{x}_j) \geq 0$.
    We can see that
    \begin{align}
        &\sum_{i=1}^m \sum_{j=1}^{m} c_i c_j K(\bm{x}_i, \bm{x}_j) \\
        &= \sum_{i=1}^m \sum_{j=1}^{m} c_i c_j \int_\mathcal{S} p(s \mid \bx_i)^\rho p(s\mid \bx_j)^\rho \mathrm{d}s \\
        &= \int_\mathcal{S}  \left \{ \sum_{i=1}^m \sum_{j=1}^m c_i c_j p(s \mid \bx_i)^\rho p(s \mid \bx_j)^\rho \right \} ds.
    \end{align}    
    Now, we focus on the inner sums under the integral, and express that 
    \begin{align}
        &\sum_{i=1}^m \sum_{j=1}^m c_i c_j p(s \mid \bx_i)^\rho p(s \mid \bx_j)^\rho \\
        &=\sum_{i=1}^m [c_i p(s \mid \bx_i)^\rho]^2 + 2\sum_{i\neq j}^m  [ c_i p(s \mid \bx_i)^\rho] [c_j p(s\mid \bx_j)^\rho ] \label{eq:ppk1rho}
        \\ &= \left [ \sum_{i=1}^m c_i p(s \mid \bx_i)^\rho \right ]^2 \geq 0. \label{eq:ppk2rho}
    \end{align}
    The last equation \eqref{eq:ppk1rho}--\eqref{eq:ppk2rho} holds due to the identity $\left ( \sum_{i=1}^m a_i \right )^2 = \sum_{i=1}^m a_i^2 + 2 \sum_{i \neq j}^m a_i a_j$, with \(a_i = c_i p(s \mid \bx_i)^\rho\). Since the integrand is non-negative for all \(s\), the entire integral is non-negative, which completes the proof.
\end{proof}
\label{eq:theorem_ppk_psd}
\end{theorem}
\end{highlight}

% \subsection{Kernel Properties}

\paragraph{Non-stationarity.} Due to the formulation in~\eqref{eq:beta_kernel}, the proposed kernel $K_{\beta}(\bm{x},\bm{x}')$ cannot be expressed solely as a function of $\bx - \bx'$, which indicates its non-stationarity. To illustrate this, we compare the Beta kernel to the Matérn kernel on the unit 1D domain in Figure~\ref{fig:1d_covar}. Whereas RBF and Matérn have constant diagonals, our Beta kernel's diagonal varies depending on $h$. 

\paragraph{Diagonal.} The diagonal of the covariance matrix w.r.t. the Beta kernel is characterized as
\begin{align}
   K_{\beta}(\bx, \bx) = \Tilde{C} \prod_{i=1}^d \frac{\Gamma(2\frac{x_i}{h_i}+1) \Gamma(2\frac{1-x_i}{h_i}+1)}{\Gamma^2(\frac{x_i}{h_i}+1)\Gamma^2(\frac{1-x_i}{h_i} + 1)}.
\end{align}

We derive an upper bound of $K_{\beta}(\bx, \bx)$ in Proposition~\ref{eq:k_bound}.
\begin{proposition} 
\label{eq:k_bound}
Assume that $h_i = h \ \forall i \in [d]$, $\forall \bx \in [0,1]^d$, we can bound $K_{\beta}(\bx, \bx)$ as follows
\begin{align}
    K_\beta(\bx, \bx) \leq 2^{3d-\frac{2d}{h}} \left (\frac{1}{h}+1 \right )^d \left (\frac{1}{h\pi}+\frac{3}{2\pi} \right )^\frac{d}{2}.
\end{align}
\begin{proof}
    The full proof is provided in Appendix~\ref{sec:proof_prop1}.
\end{proof}
\end{proposition}

\subsection{Numerical Analyses of Eigenvalue Decay}

\begin{figure}[h]
    \centering
    \croppdf{figures/eigenvalue_decay}
    \includegraphics[width=0.9\linewidth]{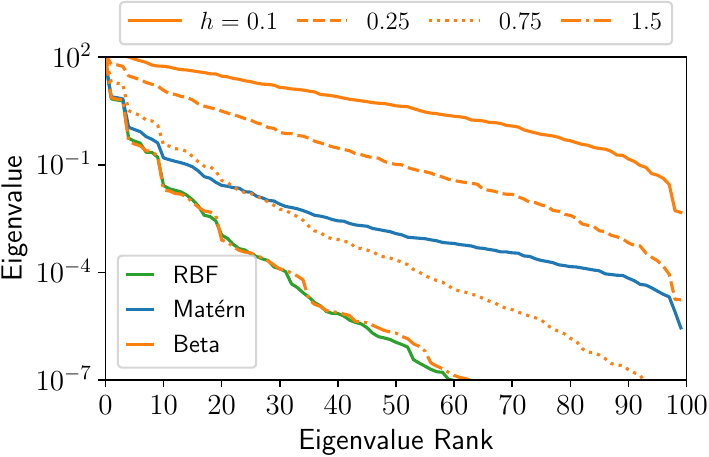}
    \caption{\small Spectral decay for the RBF ($\ell=1$), Matérn ($\nu=2.5$), and Beta kernels on 3D unit hypercube.}
    \label{fig:eigenvalue_decay}
\end{figure}

In this section, we present numerical analyses to examine the eigenvalue decay rate of the Beta kernel, which is associated with the smoothness of the functions that the GP can model. In addition, the eigendecay rate is linked to the regret bound of the kernel~\citep{srinivas2009gaussian,vakili2021information}. Our analysis is conducted with reference to the well-established eigendecay rates of the Matérn and RBF kernels, which are $O(m^{-\frac{2\nu + d}{d}})$ and $O(\exp(-m^{1/d}))$, respectively~\citep{santin2016approximation,belkin2018approximation}. 

We compute the expected spectrum decay using the following procedure: 
\begin{enumerate}
    \item We generated $300$ random data matrices $X_i$ of size $100\times d$ .
    \item For every $X_i$, we generated the corresponding kernel matrix $K_i \in \mathbb{R}^{d\times d}$ by computing $k_\ell(\cdot, \cdot)$ between all the rows of $X_i$.
    \item We computed the sorted sets of eigenvalues for all $K_i$ and then averaged them, resulting in the corresponding expected spectrum for the most important kernels, several hyperparameter settings, and problem dimensions.
\end{enumerate}

In Figure~\ref{fig:eigenvalue_decay}, we analyze the spectral decay of the RBF, Matérn, and Beta kernels on the 3D unit hypercube. Accordingly, there is a strong correlation between the bandwidth parameter $h$ and the eigenvalue decay rate of the Beta kernel. With $h \leq 0.25$, the Beta kernel shows a slower eigenvalue decay rate compared to the Matérn kernel with $\nu=2.5$. Additionally, in the logarithmic scale of Figure~\ref{fig:eigenvalue_decay}, the eigendecay of the Beta kernel appears approximately linear, which suggests a potential exponential decay rate. When $h=1.5$, its eigendecay closely matches the exponential rate of the RBF kernel.

We further conduct statistical analyses to assess whether the eigendecay of the Beta kernel follows an exponential trend. Specifically, we consider various settings with $d \in \{5, 10, 20, 50\}$ and bandwidth $h \in \{0.1, 0.25, 0.5, 0.75, 1, 1.5\}$. After computing the expected spectrum for each setting, we fit a linear regression model to examine the relationship between $\log \lambda_j$ and $j$, where $\lambda_j$ denotes the eigenvalue and $j$ is its index. If the eigendecay is exponential, this relationship should be statistically significant. As shown in Table~\ref{tab:stats_test_eigendecay}, the p-values across various settings are significantly low, providing strong evidence in support of our hypothesis.

\begin{table}[t]
    \centering
    \caption{\small P-values from the statistical analysis of the exponential eigendecay rate of the proposed kernel across different dimensions and bandwidths.}    
    \renewcommand{\arraystretch}{1.3}
    \footnotesize
    \resizebox{0.48\textwidth}{!}{%
    \begin{tabular}{ccccc}
        \toprule
        $h$ & \textbf{$d=5$} & \textbf{$d=10$} & \textbf{$d=20$} & \textbf{$d=50$} \\
        \midrule
        0.1  & 7.7$\times 10^{-36}$ & 1.4$\times 10^{-31}$ & 3.6$\times 10^{-32}$ & 1.0$\times 10^{-32}$ \\
        0.25 & 1.9$\times 10^{-43}$ & 5.6$\times 10^{-36}$ & 2.4$\times 10^{-33}$ & 5.6$\times 10^{-34}$ \\
        0.5  & 1.3$\times 10^{-40}$ & 1.1$\times 10^{-31}$ & 1.0$\times 10^{-33}$ & 9.7$\times 10^{-35}$ \\
        0.75 & 2.3$\times 10^{-37}$ & 9.7$\times 10^{-28}$ & 1.1$\times 10^{-25}$ & 5.0$\times 10^{-35}$ \\
        1    & 2.5$\times 10^{-35}$ & 2.9$\times 10^{-26}$ & 1.5$\times 10^{-21}$ & 2.4$\times 10^{-33}$ \\
        1.5  & 7.1$\times 10^{-33}$ & 5.0$\times 10^{-25}$ & 1.4$\times 10^{-19}$ & 2.3$\times 10^{-17}$ \\
        \bottomrule
    \end{tabular}}
    \label{tab:stats_test_eigendecay}
\end{table}

\section{EXPERIMENTS}
Our main focus of this section is to demonstrate fair and direct comparisons between the proposed Beta kernel and two widely-used stationary kernels -- RBF and Matérn -- in BO using GP. For that purpose, we conducted our experiments on both synthetic data and real-world vision and natural language data. Additionally, we investigated the compatibility of the mentioned kernels with a wide range of acquisition functions such as UCB, PI~\citep{kushner1964new}, EI~\citep{jones1998efficient}, corrected PI (PI\_C)~\citep{ma2019bayesian}, and corrected EI (EI\_C)~\citep{zhou2024corrected}. Furthermore, we intuitively illustrated behavioral differences of our kernel compared to RBF and Matérn.

\subsection{Experimental Setup}

\begin{table}[t]
    \centering
    \caption{\small Models and datasets for the compression task.}
    \footnotesize
    \renewcommand{\arraystretch}{1.2}
    \setlength{\tabcolsep}{2.5pt}
    \resizebox{0.48\textwidth}{!}{
    \begin{tabular}{lllrr}
        \toprule
        \textbf{Dataset} & \textbf{Task Description} & \textbf{Model} & \textbf{Params} & \textbf{$d$} \\
        \midrule
        ImageNet & Visual object classification & ViT   & 87M  & 72 \\
        SQuAD    & Question-answering           & GPT-2 & 124M & 48 \\
        SQuAD    & Question-answering           & BERT  & 109M & 72 \\
        MNLI/ RTE     & Natural language inference   & \multirow{4}{*}{DeBERTa} & \multirow{4}{*}{184M} & \multirow{4}{*}{14} \\
        QNLI     & Sentence pair classification &       &      &    \\
        MRPC     & Similarity and paraphrase    &       &      &    \\
        \bottomrule
    \end{tabular}
    }
    \label{tab:data_desc}
\end{table}

\begin{figure}[t]
    \centering
    \croppdf{figures/test_func/levy_2d}
    \subfloat[\small 2D Levy test function\label{fig:levy_2d_viz}]{\includegraphics[width=0.47\linewidth]{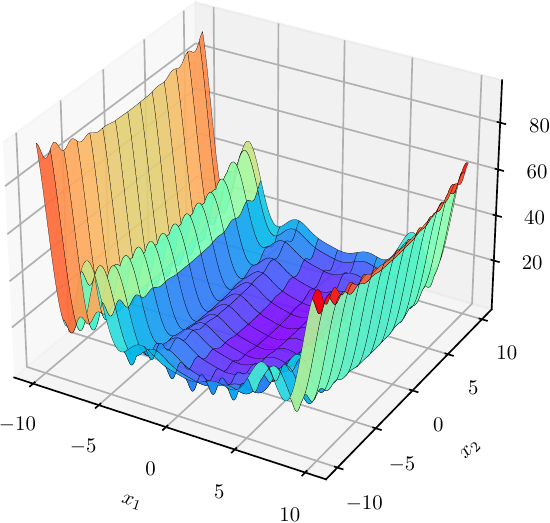}} \hfill
    \subfloat[\small Early trajectories\label{fig:levy_2d_trajectories}]{\includegraphics[width=0.52\linewidth]{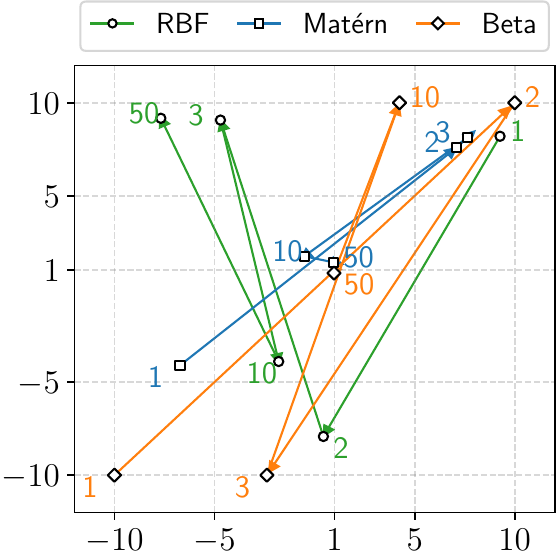}} \\
    \subfloat[\small RBF\label{fig:rbf_convergence}]{\includegraphics[width=0.32\linewidth]{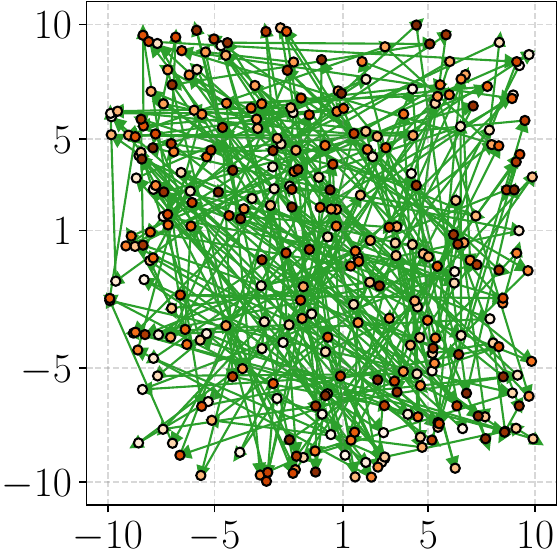}} \hfill
    \subfloat[\small Matérn\label{fig:matern_convergence}]{\includegraphics[width=0.32\linewidth]{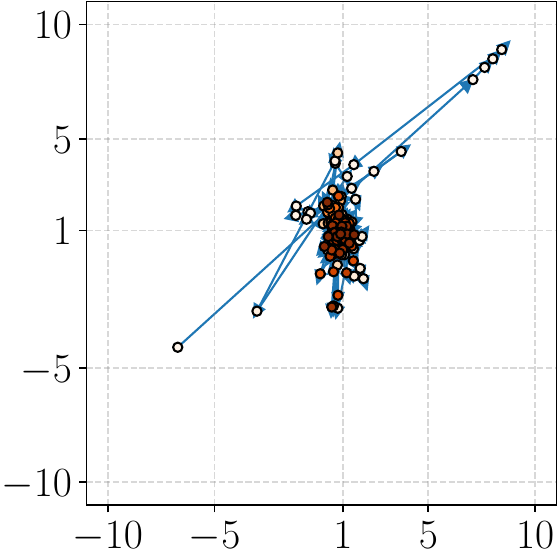}} \hfill
    \subfloat[\small Beta (ours)\label{fig:beta_convergence}]{\includegraphics[width=0.32\linewidth]{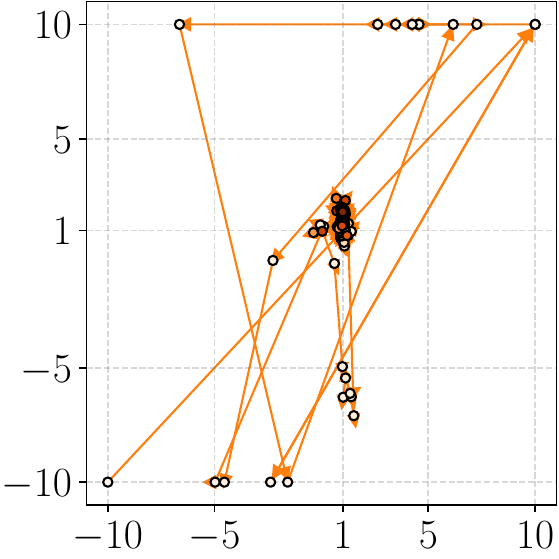}}
    \caption{\small Global optimization on the 2D Levy test function, where the global minimum is located at $(1,1)$. (b) The colored numbers represent optimization iterations corresponding to different kernel functions. (c-e) Convergence behavior over 300 iterations: (c) The RBF kernel shows over-exploration, (d) the Matérn kernel tends to neglect boundary regions, and (e) our Beta kernel achieves a more balanced trade-off between exploration and exploitation.}
    \label{fig:levy_2d}
\end{figure}

\begin{figure}[t]
    \centering
    \includegraphics[width=\linewidth]{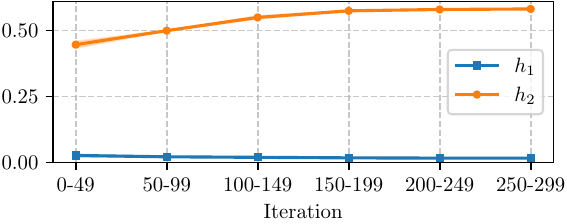}
    \caption{\small \update{Convergence of the smoothing parameters $h$ in the 2D Levy function optimization}}
    \label{fig:h_convergence}
\end{figure}

\begin{figure}[t]
    \centering    
    \subfloat[Convergence]{
    \includegraphics[width=0.98\linewidth]{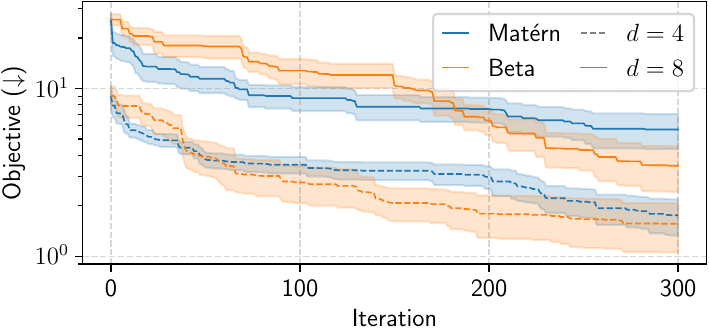}} \\
    \subfloat[Normalized distance to boundary ($d=8$)\label{fig:dist2boundary}]{
    \includegraphics[width=0.98\linewidth]{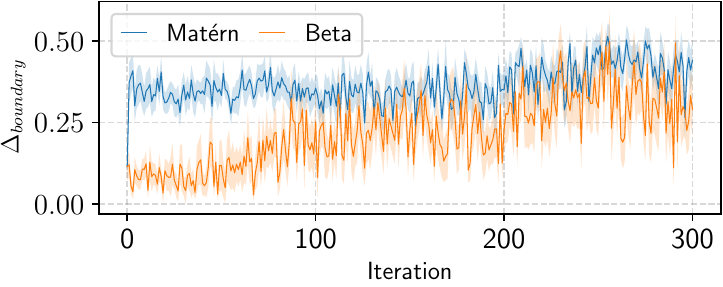}}
    \caption{\small Comparison between GP using the Matérn kernel ($\nu=2.5$), and our Beta kernel on the Levy test function.}
    \label{fig:comp_levy}
\end{figure}

\paragraph{Optimization Tasks.} For synthetic data, we performed the global optimization on the Levy test function with $d\in \{2, 4, 8\}$~\citep{laguna2005experimental}. For vision model compression, we aimed to compress the Vision Transformer (ViT) architecture~\citep{dosovitskiy2020image} on the 1k-ImageNet dataset~\citep{deng2009imagenet}. For language model compression, we performed compression on the BERT (large) model~\citep{devlin2018bert}, and the GPT-2 model~\citep{radford2019language} on the SQuAD v1 dataset~\citep{rajpurkar2016squad}. In addition, we compressed on DeBERTa-v3~\citep{he2020deberta} on four datasets -- namely MNLI~\citep{williams2017broad}, RTE~\citep{dagan2005pascal,bar2006second,giampiccolo2007third,bentivogli2009fifth}, QNLI~\citep{rajpurkar2016squad,wang2018glue}, and MRPC~\citep{dolan2005automatically} -- in the GLUE benchmark~\citep{wang2018glue}. The detailed description is presented in Table~\ref{tab:data_desc}.

\begin{table}[h]
    \centering
    \caption{\small Performance comparison across different settings ($d=20$). Setting 1: global solutions near the center. Setting 2: first global solution near the margin. Setting 3: first global solution near a vertex. \update{$\dagger$ denotes experiments performed with data warping technique from~\citep{snoek2014input}.}}
    \renewcommand{\arraystretch}{1.3}
    \resizebox{0.48\textwidth}{!}{
    \begin{tabular}{lccccc}
        \toprule
        \textbf{Kernel} & \textbf{Griewank} & \textbf{Ackley} & \textbf{Branin} & \textbf{Hartmann} & \textbf{Levy} \\
        \midrule
        \multicolumn{6}{c}{\textbf{Setting 1: Global optimum near the center}} \\
        \midrule
        RBF & 342.0±22.1 & 20.4±0.1 & 31.5±0.0 & -0.8±0.0 & 141.0±15.0 \\
        SM & 277.0±10.3 & 19.8±0.1 & 16.5±1.1 & -1.4±0.1 & 79.4±8.8 \\
        NGA  & 216.7±59.3 & 20.2±0.2 & 18.0±1.5 & -0.9±0.1 & 55.1±6.6 \\
        CYL  & 296.3±33.8 & 19.2±1.2 & 24.2±1.7 & -1.2±0.0 & 131.0±11.8 \\
        Matérn & \textbf{22.7±1.6} & \textbf{17.4±0.3} & \textbf{5.5±0.3} & \textbf{-2.4±0.1} & 10.6±1.4 \\
        Beta & 122.3±11.1 & 20.7±0.0 & 7.0±0.6 & -2.2±0.1 & \textbf{9.2±0.3} \\
        \midrule
        \multicolumn{6}{c}{\textbf{Setting 2: Global optimum near a face}} \\
        \midrule
        RBF & 326.5±14.9 & 20.3±0.1 & 22.3±2.0 & -1.1±0.0 & 104.9±8.0 \\
        SM & 279.2±7.7 & 19.5±0.2 & 18.3±1.1 & -1.4±0.1 & 81.6±5.6 \\
        NGA & 216.0±36.5 & 19.4±0.4 & 24.4±3.9 & -1.2±0.1 & 70.6±10.6 \\
        CYL & 275.9±28.0 & 20.3±0.1 & 26.2±3.0 & -1.1±0.1 & 101.8±6.7 \\
        Matérn & 24.7±2.3 & 17.3±0.3 & \textbf{4.1±0.3} & -1.6±0.0 & 11.2±1.7 \\
        Beta & \textbf{20.4±0.2} & \textbf{10.1±0.0} & 15.0±1.1 & \textbf{-2.0±0.1} & \textbf{2.2±0.1} \\
        \midrule
        \multicolumn{6}{c}{\textbf{Setting 3: Global optimum near a vertex}} \\
        \midrule
        RBF & 309.9±16.5 & 19.9±0.2 & 58.6±5.5 & -0.8±0.1 & 87.2±12.0 \\
        RBF$^{\dagger}$       & 289.5±27.7  & 19.9±0.1  & 52.1±6.2  & -0.8±0.1  & 83.2
        ±8.4 \\
        SM & 189.1±36.5 & 20.0±0.1 & 48.9±6.0 & -0.9±0.0 & 69.4±3.6 \\
        NGA & 112.3±34.7 & 19.6±0.3 & 8.5±1.7 & -1.2±0.1 & 68.4±9.6 \\
        CYL & 278.5±29.2 & 20.3±0.1 & 64.0±5.4 & - & 69.3±3.3 \\
        Matérn & 21.1±1.8 & 17.0±0.3 & 5.9±0.9 & -2.1±0.0 & 24.6±3.6 \\
        Matérn$^{\dagger}$     & 44.2±4.2    & 16.9±0.2  & 10.1±1.2  & -2.1±0.1  & 19.5±2.3 \\        
        Beta & \textbf{20.3±6.4} & \textbf{13.9±0.7} & \textbf{5.4±1.6} & \textbf{-2.2±0.1} & \textbf{6.9±0.8} \\
        \bottomrule
    \end{tabular}}
    \label{tab:merged_settings}
\end{table}

\paragraph{Implementation Details.} 
Our implementation was based on the following open-source libraries: PyTorch~\citep{paszke2019pytorch}, HuggingFace, BoTorch~\citep{balandat2020botorch}, and GPyTorch~\citep{gardner2018gpytorch}. The training processes were run on NVIDIA V100 GPUs. Each experiment was repeated $10$ times with different random seeds, then its mean and standard error were reported. The lengthscales of the RBF and Matérn kernels, along with the bandwidth of the Beta kernel, were learned through maximum marginal likelihood during training. 

\paragraph{Baselines.}
Besides the two well-known kernels RBF and Matérn, we compared our kernel to Spectral Mixture (SM)~\citep{wilson2013gaussian}, Newton-Girard Additive (NGA)~\citep{duvenaud2011additive}, and Cylindrical (CYL)~\citep{oh2018bock} kernels. 

\begin{figure}[t]
    \centering
    \croppdf{figures/objective_ImageNet_squad_comparison_UCB}        
    \includegraphics[width=\linewidth]{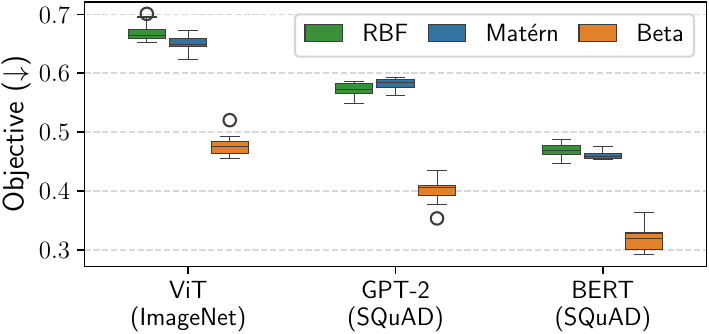}
    \caption{\small Model compression comparisons between GP using the RBF, Matérn ($\nu=2.5$), and Beta kernels on the ImageNet and SQuAD datasets. The common acquisition function is UCB.}
    \label{fig:compression_comp_imagenet_squad_ucb}
\end{figure}

\begin{table}[t]
\caption{\small Objective comparison of different kernels across the model compression tasks. \update{Each value represents the combined error rate and compression rate ($w=1$), as defined in Eq.~\eqref{eq:objective}.}}
\centering
\scriptsize
\renewcommand{\arraystretch}{1.5}
\setlength{\tabcolsep}{1.4pt}
\resizebox{0.49\textwidth}{!}{%
\begin{tabular}{lcccc}
\toprule
\multirow{2}{*}{\textbf{Kernel}} & \textbf{ViT} & \textbf{BERT} & \textbf{GPT-2} & \textbf{DeBERTa} \\
\cline{2-5}
 & ImageNet & SQuAD & SQuAD & MNLI \\
\midrule
RBF    & $0.671_{\pm 0.005}$          & $0.468_{\pm 0.005}$          & $0.571_{\pm 0.004}$          & $0.257_{\pm 0.006}$          \\
SM     & $0.665_{\pm 0.004}$          & $0.496_{\pm 0.005}$          & $0.560_{\pm 0.003}$          & $0.294_{\pm 0.004}$          \\
NGA    & $0.680_{\pm 0.002}$          & $0.483_{\pm 0.006}$          & $0.581_{\pm 0.001}$          & --                           \\
CYL    & $0.684_{\pm 0.008}$          & $0.515_{\pm 0.012}$          & $0.570_{\pm 0.005}$          & $0.304_{\pm 0.004}$          \\
Matérn & $0.651_{\pm 0.005}$          & $0.461_{\pm 0.002}$          & $0.582_{\pm 0.003}$          & $0.246_{\pm 0.008}$          \\
Beta   & $\mathbf{0.478_{\pm 0.006}}$ & $\mathbf{0.319_{\pm 0.007}}$ & $\mathbf{0.401_{\pm 0.008}}$ & $\mathbf{0.124_{\pm 0.001}}$ \\
\bottomrule
\end{tabular}}
\label{tab:kernel_comparison}
\end{table}

\begin{table}[t]
\centering
\caption{\small \update{Performance of different kernels in terms of FLOPs, F1 score, and accuracy. FLOPs saving is calculated as (original FLOPs - compressed FLOPs) / original FLOPs.}}
\subfloat[\update{FLOPs and F1 score comparison}\label{tab:flops_f1}]{
\footnotesize
\update{
    \begin{tabular}{lccc}
    \toprule
    \textbf{Kernel} & \textbf{FLOPs} & \textbf{FLOPs saving (\%)} & \textbf{F1 (\%)} \\
    \midrule
    Original & 1.5T       & 0    & \\
    RBF      & 638G  & 57.5 & 86.06 \\
    Matern   & 626G   & 58.3 & 86.47 \\
    Beta     & 349G  & 76.7 & 86.44 \\
    \bottomrule
    \end{tabular}
}}
\\
\subfloat[\update{FLOPs and accuracy comparison}\label{tab:flops_accuracy}]{
\update{
    \footnotesize
    \begin{tabular}{lccc}
    \toprule
    \textbf{Kernel} & \textbf{FLOPs} & \textbf{FLOPs saving (\%)} & \textbf{Accuracy (\%)} \\
    \midrule
    Original & 1.08T      & 0    & \\
    RBF      & 450G   & 58.4 & 74.26 \\
    Matern   & 444G   & 58.9 & 74.34 \\
    Beta     & 261G   & 75.8 & 74.29 \\
    \bottomrule
    \end{tabular}
}}
\label{tab:quantitative_results}
\end{table}

\subsection{Global Optimization on Synthetic Test Functions}

\paragraph{Behavioral Intuition on Levy Function.}
We utilized the GP-UCB algorithm for the minimization. We initially selected $3\cdot d$ data points using the Sobol's algorithm~\citep{sobol1967distribution,owen1998scrambling}, and performed the optimization in $300$ iterations.

We graphically depict the 2-dimensional Levy test function with the global minimum at $(1,1)$ in Figure~\ref{fig:levy_2d}. The convergence behaviors for different kernels are shown in Figure~\ref{fig:levy_2d_trajectories}. Due to the Beta kernel's characteristic as shown in Figure~\ref{fig:1d_covar}, our kernel prioritized exploring the two corner points $(-10, -10)$ and $(10, 10)$ as well as the boundaries of the domain. Detailed convergence trajectories are further illustrated in Figures~\ref{fig:rbf_convergence}, \ref{fig:matern_convergence}, and \ref{fig:beta_convergence}. Whereas the RBF kernel exhibited a tendency for over-exploration across the entire domain, the Matérn kernel favored the central region, albeit largely neglecting the domain boundaries. Different from the two references, our kernel with its awareness of domain boundaries maintained a reasonable ratio of exploration and exploitation. \update{In Figure~\ref{fig:h_convergence}, we illustrate the convergence of the smoothing parameters $h_1$ and $h_2$. By maximizing the marginal likelihood, these parameters were automatically learned, converging to distinct values.}

In Figure~\ref{fig:comp_levy}, we investigated the convergences of the Beta and Matérn kernels on $d$-dimensional Levy test function with $d \in \{4, 8\}$. To demonstrate the awareness of boundary regions, we computed the $L^{\infty}$-based normalized distance to the boundary, formulated as $\Delta_{boundary} = 1 - 2\|\Tilde{\bx} - \mathbf{c} \|_{\infty}$, where $\Tilde{\bx}=\left (\frac{x_1 - m_1}{M_1 - m_1}, \dots, \frac{x_d - m_d}{M_d - m_d} \right)$ is a normalized version of $\bx$, $m_i$ and $M_i$ represent the lower and upper bounds of the $i$-th dimension, and $\mathbf{c} = (\frac{1}{2}, \cdots, \frac{1}{2})$ is the center of the unit hypercube. Since the optimal solution of the Levy function lies at $(1, 1, \dots, 1)^\top$, near the domain center, Matérn demonstrated an advantage in the early iterations. The changes of $\Delta_{boundary}$ during the training 
shown in Figure~\ref{fig:dist2boundary} quantitatively indicate the awareness of the boundary regions of our kernel, which is consistent with the observation in Figure~\ref{fig:levy_2d} and Table~\ref{tab:merged_settings}. 
The higher dimensional the domain was, the more iterations were needed to explore the domain boundaries. After exploring the hypercube's boundaries the Beta kernel regained its performance in the later stages. 

\paragraph{Impact of Optima Location on GP Performance.}
\label{sc:unit_hypercube_ex}
We considered the unit hypercube with $d=20$. Let $\varepsilon$ be a margin, we could split the unit hypercube into three partitions: (i) central volume, denoted by $V_c$ with $|V_c|=(1 - 2\varepsilon)^d$, (ii) $\varepsilon$-size $2^d$ sub-hypercubes near vertices, represented by $V_v$ with $|V_v|=(2\varepsilon)^d$, (iii) the remaining volumes near the faces, denoted by $V_f$ with $|V_f| = 1 - (1-2\varepsilon)^d - (2\varepsilon)^d$. By fixing $\varepsilon=0.05$, for instance, then we had $|V_c|\approx 0.1216$, $|V_v|=10^{-20}$, and $|V_f|\approx0.8784$. Originally, the synthetic test functions -- Griewank, Ackley, repeated Branin, repeated Hartmann6, and Levy -- have optimal solutions near the center of the unit hypercube (setting 1). To validate the impact of optima location on the optimization, we cropped the domains of those functions to enforce one of the optima close to a face (setting 2) or a vertex (setting 3) of the unit hypercube. Specifically, we shifted $x_1^{\text{min}}$ such that $x_1^* - x_1^{\text{min}} = \varepsilon (x_1^{\text{max}} - x_1^{\text{min}})x^*$ for setting 2. For setting 3, we ensured that $x_i^* - x_i^{\text{min}} = \varepsilon (x_i^{\text{max}} - x_i^{\text{min}}), \forall i \in [d] $, where $x^*$ is the first optimum.

We present the results of the three settings in Table~\ref{tab:merged_settings}. Of all settings, the Matérn kernel was the most competitive baseline. In setting 1, Matérn achieved the highest objective values in 4 out of 5 test functions, while our kernel was the second-best in most of the cases. The proposed kernel notably showed its strength in settings 2 and 3 by outperforming all the baselines in 9 out of 10 test functions.

\begin{figure}[t]
    \centering
    \croppdf{figures/objective_comparison_tasks_GLUE}        
    \includegraphics[width=\linewidth]{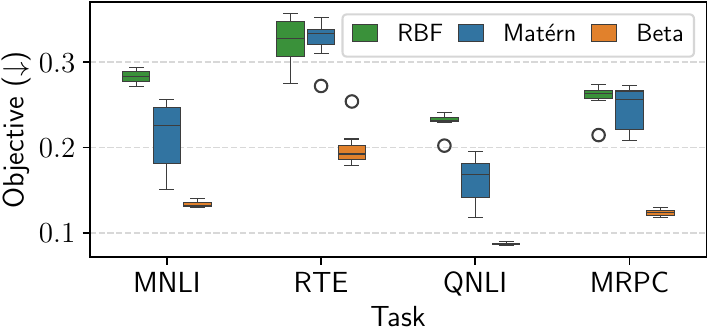}
    \caption{\small DeBERTa-v3 compression comparison between GP using the RBF kernel, the Matérn kernel ($\nu=2.5$), and our Beta kernel on different tasks in the GLUE benchmark ($d=14$). The common acquisition function is UCB.}
    \label{fig:compression_comp_glue_ucb}
\end{figure}

\subsection{Vision and Language Model Compression}
\paragraph{Compression Task Description.} 
We utilized the LoSparse method~\citep{li2023losparse}, which was originally designed to compress linear layers in transformer-based language models using low-rank and sparse approximation. A linear layer is expressed as $\mathbf{Y}=\mathbf{X}\mathbf{W}$ where $\mathbf{W} \in \mathbb{R}^{d_1 \times d_2}$ is learnable parameters, $\mathbf{X} \in \mathbb{R}^{1 \times d_1}$ is input and output features, and $\mathbf{Y} \in \mathbb{R}^{1 \times d_2}$ is the output. LoSparse performs both low-rank decomposition and sparse approximation on $\mathbf{W}$, formulated as $\mathbf{W}=\mathbf{U}\mathbf{V} + \mathbf{S}$ with $\mathbf{U} \in \mathbb{R}^{d_1 \times r}$, $\mathbf{V} \in \mathbb{R}^{r \times d_2}$, and $\mathbf{S} \in \mathbb{R}^{d_1 \times d_2}$. In~\citep{li2023losparse}, the low ranks $r$ of all linear layers were identical and independent of layer indices. In our work, we utilized BO with GP to search for the global optimal set of ranks for all linear layers. Precisely, the model compression is formulated as a multiple objective optimization as follows
\begin{align}
    \bx^{*} = \argmin_{\bx \in [0,1]^d} \left [ w \cdot \mathcal{R}(\bx) + \mathcal{L}(\bx) \right ],
    \label{eq:objective}
\end{align}
where $\bx$ represents the $d$-dimensional vector of low ranks, $\mathcal{R}(\cdot)$ is the compression rate compared to the original model, $\mathcal{L}(\cdot)$ is the error rate of the compressed model, and $w$ is a positive coefficient. 

In practice, we constrained $x_i \in [0.05, 0.95], \forall i \in [d]$, and set $w=1$. We initially sampled $5$ data points using the Sobol's algorithm~\citep{sobol1967distribution,owen1998scrambling}, and performed the optimization in $30$ iterations. After obtaining each compressed version of a model, we fine-tuned it for a varying number of iterations based on the evaluation cost (see Table~\ref{tab:model_training}). 
% \paragraph{Comparisons to RBF and Matérn kernels.} 

\begin{figure}[t]
    \centering
    \croppdf{figures/objective_comparison_acqs_GLUE}
    \includegraphics[width=\linewidth]{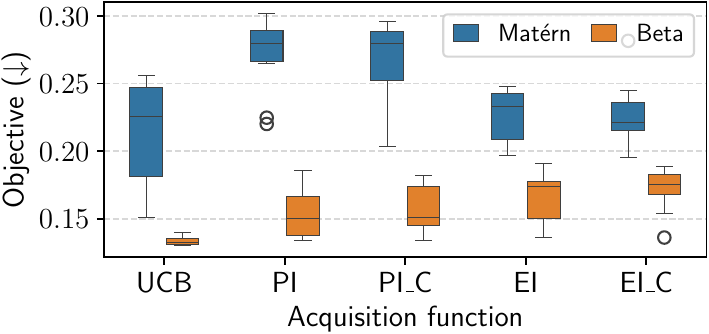}
    \caption{\small DeBERTa-v3 compression ($d=14$) comparison between GP using the Matérn kernel ($\nu=2.5$) and our Beta kernel on the MNLI task of GLUE.}
    \label{fig:compression_comp_mnli_glue}
\end{figure}

\paragraph{ViT, GPT-2, and BERT \update{Compression}.} 
For this set of experiments, we employed UCB as the common acquisition function. In Figure~\ref{fig:compression_comp_imagenet_squad_ucb} and Table~\ref{tab:kernel_comparison}, we present the compression results on the ViT, GPT-2, and BERT models. Overall, Matérn was the most competitive baseline across the tasks (see Table~\ref{tab:kernel_comparison}). In Figure~\ref{fig:compression_comp_imagenet_squad_ucb}, while the difference between RBF and Matérn was insignificant, the Beta kernel enabled BO with GP to achieve substantially better compression objectives compared to the two well-known baselines across the three settings. Specifically, when compressing ViT on ImageNet, utilizing the Beta kernel yielded an objective of $0.478\pm0.006$, which was $0.193$ and $0.172$ better than RBF and Matérn, respectively. For the GPT-2 compression task, our Beta kernel achieved an objective of $0.401\pm0.008$, outperforming RBF and Matérn with margins of $0.170$ and $0.181$, respectively. For BERT compression, our method reached an objective of $0.319\pm0.007$, substantially surpassing RBF and Matérn by $0.147$ and $0.142$, respectively.

\begin{figure*}[t]
    \centering    
    % Cropped PDFs
    \croppdf{figures/density_RBF_ImageNet_3d}
    \croppdf{figures/density_matern_ImageNet_3d}
    \croppdf{figures/density_BUC_ImageNet_3d}
    \croppdf{figures/density_RBF_squad_3d}
    \croppdf{figures/density_matern_squad_3d}
    \croppdf{figures/density_BUC_squad_3d}
    \croppdf{figures/density_RBF_GLUE_3d}
    \croppdf{figures/density_matern_GLUE_3d}
    \croppdf{figures/density_BUC_GLUE_3d}
    \begin{minipage}{0.60\linewidth}
    % Include cropped PDFs
    \subfloat[\small RBF]{
    \includegraphics[width=0.33\linewidth]{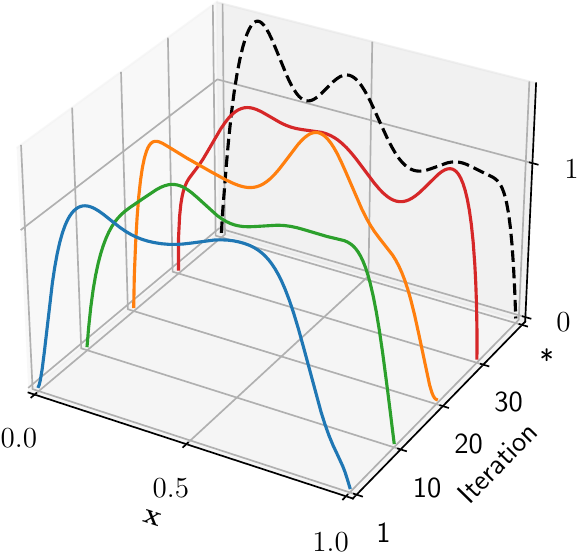}} 
    % \hfill
    \subfloat[\small Matérn]{
    \includegraphics[width=0.33\linewidth]{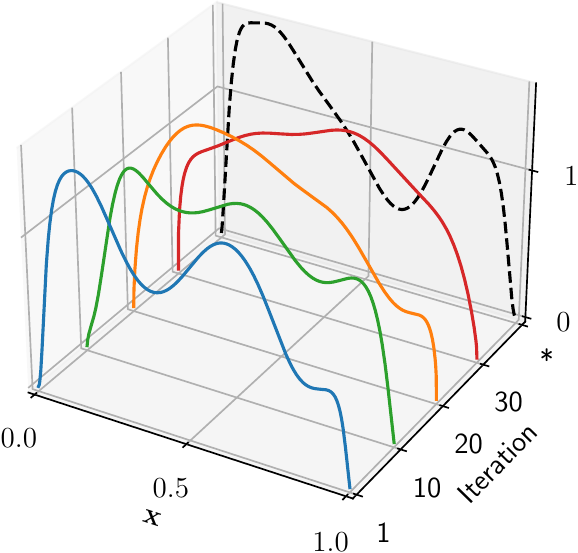}} 
    % \hfill
    \subfloat[\small Beta]{
    \includegraphics[width=0.33\linewidth]{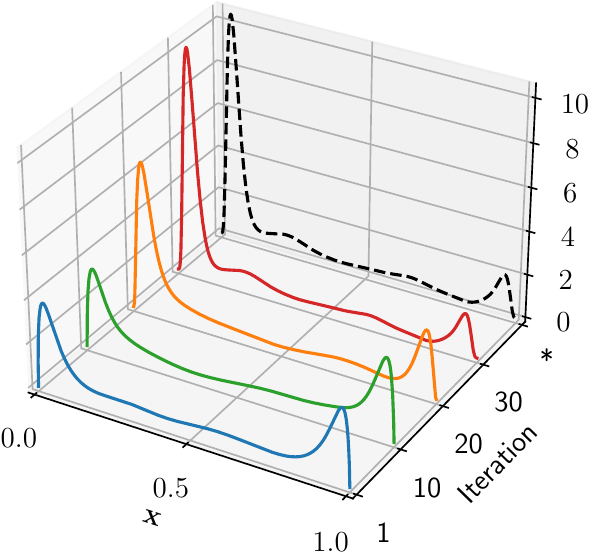}} \\
    \subfloat[\small RBF]{
    \includegraphics[width=0.33\linewidth]{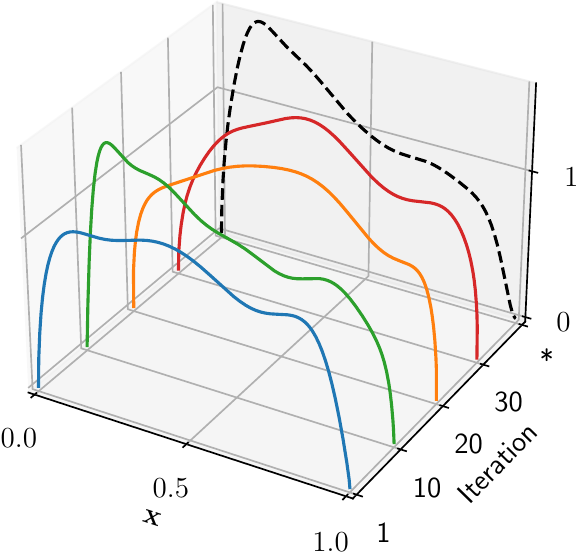}} 
    \subfloat[\small Matérn]{
    \includegraphics[width=0.33\linewidth]{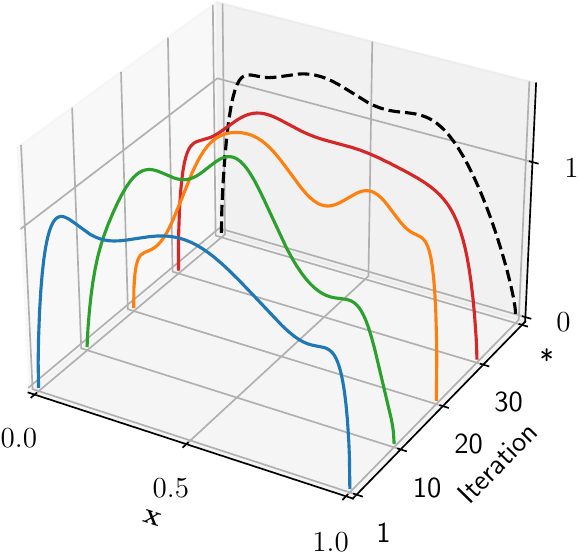}}     
    \subfloat[\small Beta]{
    \includegraphics[width=0.33\linewidth]{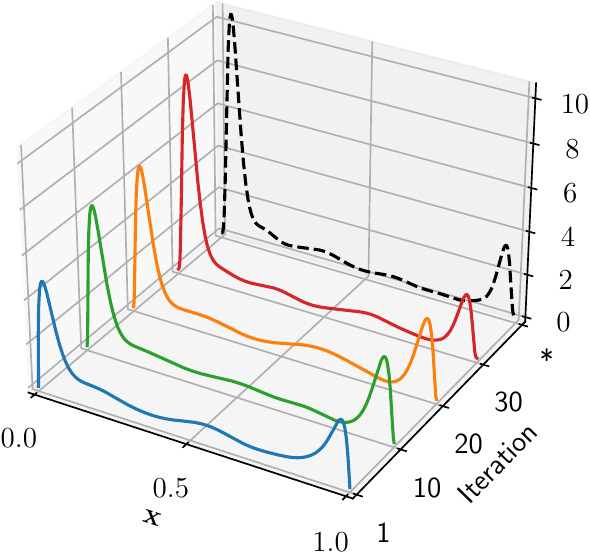}}      
    \end{minipage}%
    \hfill
    \begin{minipage}{0.38\textwidth}
        \subfloat[\small ViT\label{fig:viz_compression_layers_vit}]{
    \includegraphics[width=0.48\linewidth]{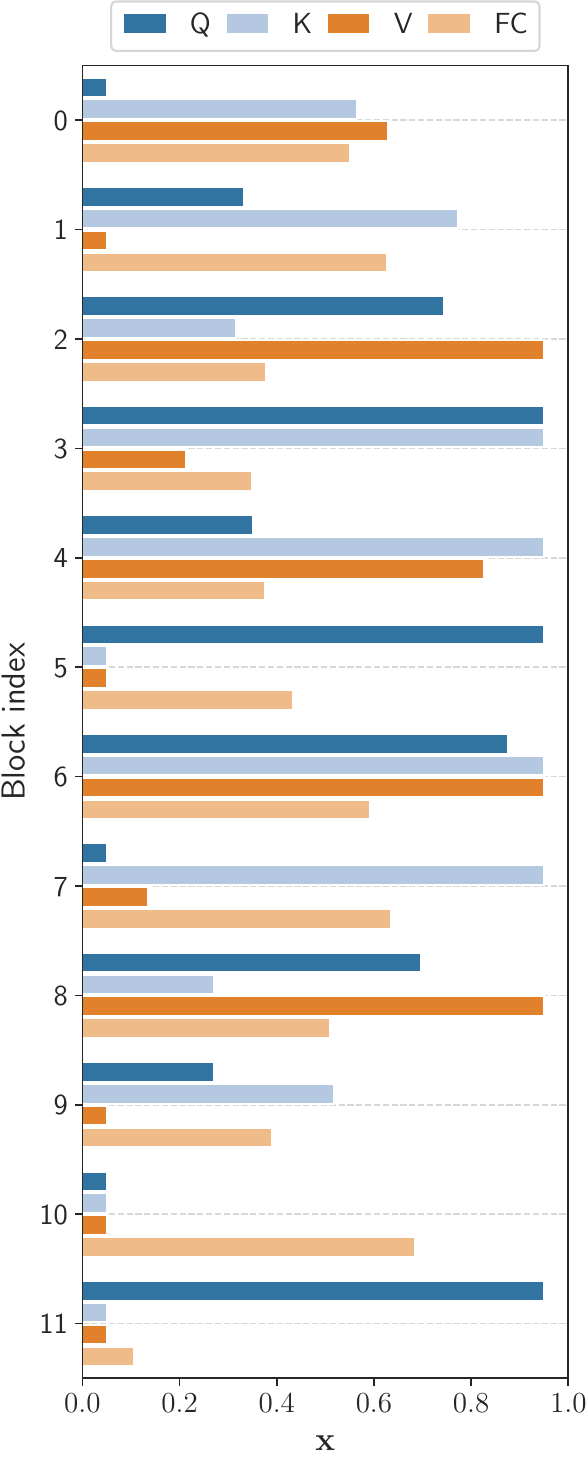}} \hfill
    \subfloat[\small BERT\label{fig:viz_compression_layers_bert}]{
    \includegraphics[width=0.48\linewidth]{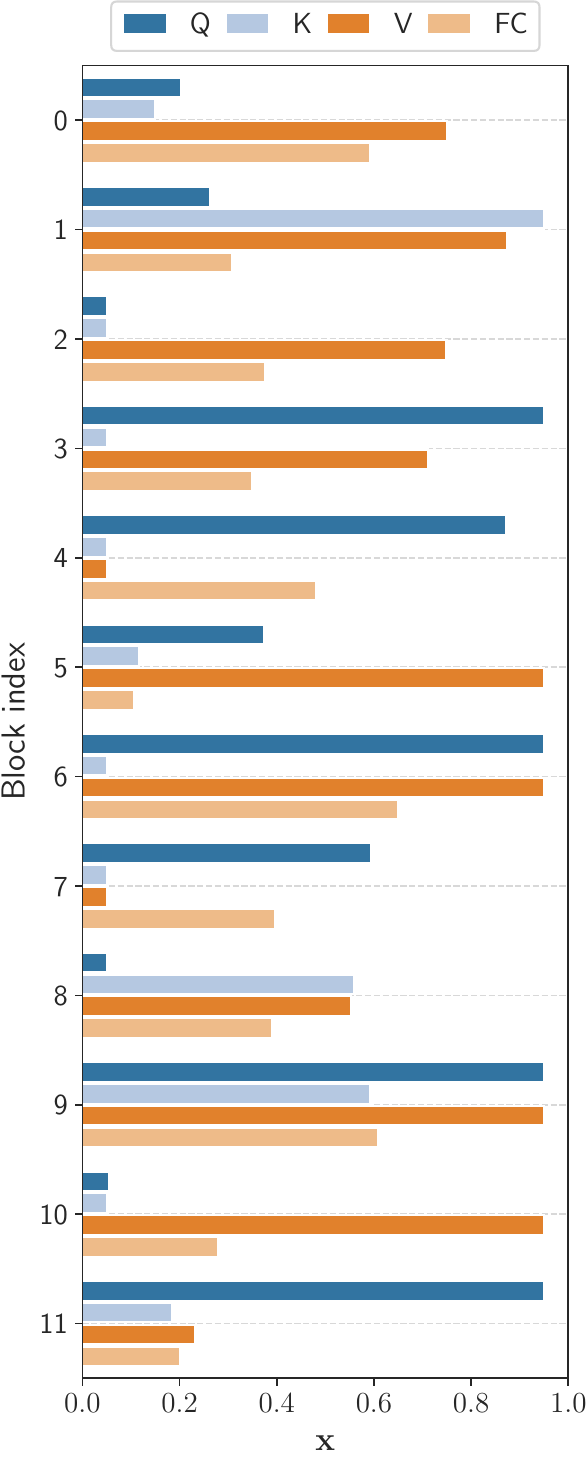}}
    \end{minipage}    
    \caption{\small (a-f) Convergence comparison between different kernels. The rows (a-c) and (d-f) correspond to the ViT, and BERT compression tasks, respectively. The blue, green, orange, and red curves indicate the density estimates at iterations 1, 10, 20, and 30, respectively. $*$ represents $\arg \min_{t \in [T]} f(\bm{x}_t)$ with $T=30$, and the dashed curves are the corresponding density estimates. (g-h) Detailed compression results using the Beta kernel. Q, K, V, and FC indicate ``query'', ``key'', ``value'', and fully connected layers, respectively.}
    \label{fig:compression_x_iter_viz}
\end{figure*}

In Table~\ref{tab:quantitative_results}, we further analyzed individual metrics in the optimization objective. In both tasks, the advantage of the Beta kernel primarily came from the compression rate improvement. Compared to Matérn, our kernel maintained insignificant trade-offs of performance for the gains of $17.3\%$ and $14.09\%$ in compression rate on ViT and BERT, respectively. Our compressed models were $14.79$ and $43.68$ times more computationally efficient than the original ViT and BERT, respectively.

\paragraph{DeBERTa-v3 Compression.} The results of compressing DeBERTa-v3 on the four tasks from the GLUE benchmark are presented in Figure~\ref{fig:compression_comp_glue_ucb} and part of Table~\ref{tab:kernel_comparison}. On the MNLI dataset, the Matérn kernel was the strongest baseline competitor. The proposed Beta kernel achieved objective scores of $0.124\pm0.001$, $0.198\pm0.007$, $0.087\pm0.000$, and $0.133\pm0.001$ on MNLI, RTE, QNLI, and MRPC, respectively, which consistently outperformed both RBF and Matérn kernels. The objective gaps between our kernel and the Matérn kernel were $0.08$, $0.127$, $0.075$, and $0.123$ on MNLI, RTE, QNLI, and MRPC, respectively.

We further examined the combinations of the Matérn and Beta kernels against various acquisition functions in the MNLI task (see Figure~\ref{fig:compression_comp_mnli_glue}). Overall, the Beta kernel achieved substantially better objective scores than the baseline across all five acquisition functions. While UCB and corrected EI were the most compatible acquisition functions for the Matérn kernel, our kernel demonstrated a clear advantage when paired with UCB. Specifically, the Beta kernel with UCB obtained an objective score of $0.133$, outperforming its combinations with PI, corrected PI, EI, and corrected PI by $0.021$, $0.023$, $0.033$, and $0.038$, respectively. When using UCB, the Beta kernel outperformed the Matérn kernel by $0.08$.

\paragraph{Learned Compression Strategy Analysis.} 

In Figure~\ref{fig:compression_x_iter_viz}, we graphically demonstrate the kernel density estimates (KDEs) of $\{x_i\}_{i \in [d]}$ during training. To properly address the boundary issue of KDE on these bounded domains, we employed the Beta KDE~\citep{chen1999beta}. On both the vision and language models, we observed that the two reference kernels tended to moderately compress all linear layers. As a result, the mean of each estimated density was centered around $0.5$. In contrast, the Beta kernel selected a small set of relevant layers, applying slight compression to them while strongly suppressing the irrelevant ones (see Figures~\ref{fig:viz_compression_layers_vit} and \ref{fig:viz_compression_layers_bert}). 
Given that our kernel's compression outperformed the two baseline kernels in Table~\ref{tab:quantitative_results}, it implies that the optimal solution to model compression may lie near the boundaries of the unit hypercube, which played to the strengths of the Beta kernel.

\section{RELATED WORK}

\paragraph{RBF and Matérn Kernels.} Various studies have been conducted to derive the regret bounds of the RBF and Matérn kernels~\citep{srinivas2009gaussian,scarlett2017lower,scarlett2018tight,belkin2018approximation,santin2016approximation,cai2021lower,vakili2021information}. The key distinction between the two kernels lies in their rates of eigenvalue decay (termed eigendecay). The RBF kernel exhibits an exponential decay of eigenvalues, while the Matérn kernel's eigenvalues decay at a polynomial rate~\citep{vakili2021information} (see Figure~\ref{fig:eigenvalue_decay}). In this work, we numerically show that the proposed Beta kernel's eigendecay rate is exponential, which is similar to RBF's. 

\paragraph{Non-stationary Kernels.} \cite{higdon1999non} introduced a spatially evolving family of smoothing kernels on $\mathbb{R}^2$. \cite{paciorek2003nonstationary} extended the Matérn kernel into a non-stationary version on $\mathbb{R}^d$. \cite{remes2017non} proposed the non-stationary generalized spectral mixture kernel with input-dependent GP frequency surfaces. A common characteristic of these non-stationary kernels is that they are all defined on unbounded domains. In contrast, our Beta kernel is constructed from products of Beta distributions, which allows it to naturally capture a wide variety of smooth functions within bounded unit hypercubes. Such a formulation gives the Beta kernel a distinct advantage in being more sensitive to boundary regions.

\section{CONCLUSION}
We present a novel non-stationary kernel constructed from products of Beta probability density functions, whose close form is a product of Gamma functions. 
We provide empirical evidence indicating that the proposed kernel's eigendecay rate is exponential. Such an eigendecay rate is similar to RBF's, which was proved to have sub-linear regret bound by \cite{vakili2021information}. \update{However, a primary limitation of our study is the absence of a formal regret bound for the proposed kernel. Future work should prioritize deriving these theoretical guarantees to deepen our understanding of the algorithm’s performance and enhance its robustness.} Our experiments indicate that the Beta kernel is robust in modeling functions with optima near faces or vertices of the unit hypercube. We show that our kernel substantially outperforms the two well-known kernels -- RBF and Matern -- on synthetic data as well as the model compression tasks. \update{Our codebase is made publically available at~\url{https://github.com/imedslab/BetaKernel}}.

\begin{highlight}
\section*{Acknowledgments}
The authors wish to acknowledge CSC—IT Center for Science, Finland, for generous computational resources. H.Z. and M.B.B. acknowledge support  from the Flemish Government (AI Research Program) and the Research
Foundation - Flanders (FWO) through project number
G0G2921N. H.Z. is supported by the China Scholarship
Council. A.T. and H.H.N. were supported by the Research Council of Finland (Profi6 336449 funding program), Sigrid Juselius foundation, University of Oulu strategic funds, and the European Union Horizon Programs (STAGE project, decision 101137146, and UltraSense project, Grant Agreement number 101130192). Finally, we gratefully acknowledge Prof. Simo Särkkä for insightful discussions on regret bounds and probability product kernels.
\end{highlight}

\bibliography{ref}

\clearpage

\onecolumn

\appendix

% \maketitle
% \newtheorem{lemma}[theorem]{Lemma}
\renewcommand{\thepage}{S\arabic{page}} 
\renewcommand{\thetable}{S\arabic{table}}  
\renewcommand{\thefigure}{S\arabic{figure}}     

\setcounter{page}{1}
\setcounter{figure}{0}
\setcounter{table}{0}
\setcounter{section}{0}
\setcounter{equation}{0}
\setcounter{lemma}{0}

\title{Bayesian Optimization over Bounded Domains with the Beta Product Kernel\\(Supplementary Material)}

\maketitle

\begin{table}[h]
    \centering
    \caption{Training details of different models.}
    \begin{tabular}{|l|c|c|c|c|}
        \hline
        \textbf{Model} & \textbf{Training Steps} & \textbf{Batch Size} & \textbf{Time per Evaluation} & \textbf{Num of Training Samples} \\
        \hline
        GPT-2 & 512 & 25 & 55 mins & 12,800 \\
        BERT & 256 & 100 & 40 mins & 25,600 \\
        DeBERTa-v3 & 256 & 100 & 30 mins & 25,600 \\
        ViT & 1 epoch & 200 & 7 mins & Full training set \\
        \hline
    \end{tabular}    
    \label{tab:model_training}
\end{table}

% \begin{comment}
\section{Proof of upper bound of $K_{\beta}(\bx, \bx)$}
\label{sec:proof_prop1}
\subsection{Proof of Lemma 1}
\label{sec:proof_lemma1}
\begin{lemma} 
\label{lemma1}
For gamma function $\Gamma(\cdot)$, we have that
\begin{equation} \footnotesize
    \frac{\Gamma(2x+1)}{\Gamma^2(x+1)} = \frac{2^{2x} \Gamma(x+\frac{1}{2})}{\sqrt{\pi} \Gamma(x+1)} 
\end{equation}  
\end{lemma}

\begin{proof}
    We start by utilizing the duplication formula for the Gamma function, which states:
\[
\Gamma(z) \Gamma\left(z + \frac{1}{2}\right) = 2^{1-2z} \sqrt{\pi} \Gamma(2z)
\]
Applying the duplication formula to our specific case by setting \( z = x + \frac{1}{2} \), we get:
\[
\Gamma\left(x + \frac{1}{2}\right) \Gamma(x+1) = 2^{-2x} \sqrt{\pi} \Gamma(2x+1).
\]
Dividing both sides by \(\Gamma^2(x+1)\), we obtain:
\[
\frac{\Gamma(2x+1)}{\Gamma^2(x+1)} = \frac{2^{2x} \Gamma\left(x + \frac{1}{2}\right)}{\sqrt{\pi} \Gamma(x+1)}.
\]
which completes the proof.
\end{proof}

\subsection{Proof of Lemma 2}
\label{sec:proof_lemma2}
\begin{lemma} \label{prop:gramma2_revised}
For $x \geq 0 $ and $0 < s < 1$, it holds that:
\[
\left (\frac{2}{2x + 1} \right )^{\frac{1}{2}} \leq \frac{\Gamma(x+\frac{1}{2})}{\Gamma(x+1)} \leq  2.
\]
\begin{proof}
    The Wendel's inequality is stated as 
\begin{align}
    \left ( \frac{z}{z+s} \right )^{1 - s} \leq \frac{1}{z^s} \cdot \frac{\Gamma(z+s)}{\Gamma(z)} \leq 1,
\end{align}
where $z > 0$ and $s \in (0,1)$. When $s=\frac{1}{2}$, it is equivalent to 
\begin{align}
1 &\leq z^{\frac{1}{2}} \frac{\Gamma(z)}{\Gamma(z+ \frac{1}{2})} \leq  \left ( \frac{z + \frac{1}{2}}{z}\right )^{\frac{1}{2}} \\
    z^{-\frac{1}{2}} &\leq \frac{\Gamma(z)}{\Gamma(z+ \frac{1}{2})} \leq z^{- \frac{1}{2}} \left ( \frac{1}{2z} + 1 \right )^{\frac{1}{2}}.
\end{align}
Apply the inequality with $z=x+\frac{1}{2}$ and $s=\frac{1}{2}$, we have that
\begin{align}
    \left (\frac{2}{2x + 1} \right )^{\frac{1}{2}} \leq \frac{\Gamma\left(x + \frac{1}{2}\right)}{\Gamma\left(x + 1\right)} &\leq \left(x + \frac{1}{2}\right)^{-\frac{1}{2}} \left( \frac{1}{2x + 1} + 1 \right)^{\frac{1}{2}} \\
    &= \left [ \frac{2}{2x + 1} \left ( \frac{1}{2x+1} + 1\right ) \right]^{\frac{1}{2}} =: g(x).
\end{align}
As $g(x)$ is a decreasing function on $[0, \infty)$, thus $g(0) = \max_{x\geq0}g(x) = 2$, which concludes the proof.
\end{proof}
\end{lemma}

\subsection{Proof of Proposition 1}
\begin{proposition}    
    $K_{\beta}(\bx, \bx) = \mathcal{O}(2^{2d - \frac{2d}{h}} h^{-\frac{3d}{2}}), \forall \bx \in [0,1]^d$.
\end{proposition}
\begin{proof}
For any $\bx=(x_1, \dots, x_d) \in [0,1]^d$, $K_{\beta}(\bx, \bx)$ is expressed as
\begin{align}
   K_{\beta}(\bx, \bx) = C \prod_{i=1}^d \frac{\Gamma(2\frac{x_i}{h_i}+1) \Gamma(2\frac{1-x_i}{h_i}+1)}{\Gamma^2(\frac{x_i}{h_i}+1)\Gamma^2(\frac{1-x_i}{h_i} + 1)}, \nonumber
\end{align}
where
\begin{align}
    C = \frac{\Gamma^2(\frac{1}{h_i}+2)}{\Gamma(\frac{2}{h_i}+2)}.
\end{align}

By utilizing Lemma~\ref{lemma1} with $\frac{x_i}{h_i}$ and $\frac{1-x_i}{h_i}$, we have that
\begin{align}
    K_{\beta}(\bx, \bx) &= C \prod_{i=1}^d \frac{2^{2\frac{x_i}{h_i} + 2(1 - \frac{x_i}{h_i})} \Gamma(\frac{x_i}{h_i}+ \frac{1}{2}) \Gamma(\frac{1-x_i}{h_i} + \frac{1}{2})}{\pi\Gamma(\frac{x_i}{h_i}+1)\Gamma(\frac{1-x_i}{h_i} + 1)} \\
    &=C \prod_{i=1}^d \frac{4}{\pi}\frac{\Gamma(\frac{x_i}{h_i}+ \frac{1}{2}) \Gamma(\frac{1-x_i}{h_i} + \frac{1}{2})}{\Gamma(\frac{x_i}{h_i}+1)\Gamma(\frac{1-x_i}{h_i} + 1)} 
\end{align}

When applying Lemma~\ref{prop:gramma2_revised}, we have that
\begin{equation}
    \begin{aligned}
        \frac{\Gamma(\frac{x_i}{h_i}+\frac{1}{2}) \Gamma(\frac{1-x_i}{h_i}+\frac{1}{2})}{ \Gamma(\frac{x_i}{h_i}+1)\Gamma(\frac{1-x_i}{h_i} + 1)} \leq 4
    \end{aligned}
\end{equation}

Therefore, we can derive that
\begin{align}
        K_{\beta}(\bx, \bx) &\leq \frac{16^d}{\pi^d} \prod_{i=1}^d \frac{\Gamma^2(\frac{1}{h_i}+2)}{\Gamma(\frac{2}{h_i}+2)} \\
        &= \frac{16^d}{\pi^d} \prod_{i=1}^d \frac{(\frac{2}{h_i}+2) \Gamma^2(\frac{1}{h_i}+2)}{\Gamma(\frac{2}{h_i}+2+1)} && \text{ by } \Gamma(x+1)=x\Gamma(x)\\
        &= \frac{16^d}{\pi^d} \prod_{i=1}^d \frac{(\frac{2}{h_i}+2)\sqrt{\pi }\Gamma(\frac{1}{h_i}+2)}{2^{2(\frac{1}{h_i}+1)}\Gamma(\frac{1}{h_i}+\frac{3}{2})} && \text{ by Lemma}~\ref{lemma1} \text{ with } x=\frac{1}{h_i}+1\\
        & = \frac{2^{2d}}{\pi^{\frac{d}{2}}} \prod_{i=1}^d \frac{(\frac{2}{h_i}+2)\Gamma(\frac{1}{h_i}+2)}{2^{\frac{2}{h_i}}\Gamma(\frac{1}{h_i}+\frac{3}{2})} \\
        &\leq \frac{2^{3d}}{\pi^{\frac{d}{2}}} \prod_{i=1}^d \frac{(\frac{1}{h_i}+1)(\frac{1}{h_i}+ \frac{3}{2})^{\frac{1}{2}} }{2^{\frac{2}{h_i}}} && \text{ by Lemma}~\ref{prop:gramma2_revised} \ \frac{\Gamma(\frac{1}{h_i}+2)}{\Gamma(\frac{1}{h_i}+
    \frac{3}{2})} \le \left (\frac{1}{h_i}+ \frac{3}{2} \right )^{\frac{1}{2}} \\
        &=  2^{3d-\frac{2d}{h}} \left (\frac{1}{h}+1 \right )^d \left (\frac{1}{h\pi}+\frac{3}{2\pi} \right )^\frac{d}{2}     
\end{align}

Therefore, we conclude the proof.
\end{proof}
% \end{comment}
\end{document}